\documentclass{article}
\usepackage{ijcai23}

\usepackage{times}
\usepackage{soul}
\usepackage{url}
\usepackage[hidelinks]{hyperref}
\usepackage[utf8]{inputenc}
\usepackage[small]{caption}
\usepackage{graphicx}
\usepackage{amsmath}
\usepackage{amsthm}
\usepackage{booktabs}
\usepackage{algorithm}
\usepackage{algorithmic}
\usepackage[switch]{lineno}
\usepackage{amssymb}
\usepackage{tikz}
\usetikzlibrary{bayesnet}
\usetikzlibrary{arrows}
\usepackage{subcaption}
\usepackage{xcolor}
\usepackage[T1]{fontenc}

\newtheorem{theorem}{Theorem}

\title{Transferable Curricula through Difficulty Conditioned Generators}

\author{
Sidney Tio
\and
Pradeep Varakantham
\affiliations
Singapore Management University\\
\emails
sidney.tio.2021@phdcs.smu.edu.sg, pradeepv@smu.edu.sg
}

\begin{document}

\maketitle

\begin{abstract}
  Advancements in reinforcement learning (RL) have demonstrated superhuman performance in complex tasks such as Starcraft, Go, Chess etc. However, knowledge transfer from Artificial ``Experts" to humans remain a significant challenge. A promising avenue for such transfer would be the use of curricula. Recent methods in curricula generation focuses on training RL agents efficiently, yet such methods rely on surrogate measures to track student progress, and are not suited for training robots in the real world (or more ambitiously humans). In this paper, we introduce a method named \textit{Parameterized Environment Response Model} (PERM) that shows promising results in training RL agents in parameterized environments. Inspired by Item Response Theory, PERM seeks to model difficulty of environments and ability of RL agents directly. Given that RL agents and humans are trained more efficiently under the ``zone of proximal development", our method generates a curriculum by matching the difficulty of an environment to the current ability of the student.  In addition, PERM can be trained offline and does not employ non-stationary measures of student ability, making it suitable for transfer between students. We demonstrate PERM's ability to represent the environment parameter space, and training with RL agents with PERM produces a strong performance in deterministic environments. Lastly, we show that our method is transferable between students, without any sacrifice in training quality.
\end{abstract}

\section{Introduction}

Consider the education of calculus. We know that there is a logical progression in terms of required knowledge before mastery in calculus can be achieved: knowledge of algebra is required, and before that, knowledge of arithmetic is required. While established progressions for optimal learning exists in education, they often require extensive human experience and investment in curriculum design. Conversely, in  modern video games, mastery requires hours of playthroughs and deliberate learning with no clear pathways to progression. In both cases, a coach or a teacher, usually an expert, is required to design such a curriculum for optimal learning. Scaffolding this curriculum can be tedious and in some cases, intractable. More importantly, it requires deep and nuanced knowledge of the subject matter, which may not always be accessible.

The past decade has seen an explosion of Reinforcement Learning (RL, \cite{sutton1998introduction}) methods that achieve superhuman performance in complex tasks such as DOTA2, Starcraft, Go, Chess, etc. (\cite{berner2019dota}, \cite{arulkumaran2019alphastar}, \cite{silver2016mastering}, \cite{silver2017mastering} Given the state-of-the-art RL methods, we propose to explore methods that exploit expert-level RL agents for knowledge transfer to humans and to help shortcut the learning process. One possible avenue for such a transfer to take place would be the use of curricula.

Recent methods in curricula generation explores designing a curricula through \textit{Unsupervised Environment Design} (UED, \cite{dennis2020emergent}). UED formalizes the problem of finding adaptive curricula in a teacher-student paradigm, whereby a teacher finds useful environments that optimizes student learning, while considering student's performance as feedback. While prior work in UED (e.g. \cite{parker2022evolving}, \cite{du2022takes}, \cite{li2023diversity}, \cite{li2023effective}) has trained high-performing RL students on the respective environments, these method rely on surrogate objectives to track student progress, or co-learn with another RL agent (\cite{dennis2020emergent}, \cite{du2022takes}, \cite{parker2022evolving}), both of which would be impractical for transfer between students (artificial or agents in real world settings alike).

For transfer between students, We require methods that do not use additional RL students, or are able to directly track student's learning progress. In this work, we introduce the Item Response Theory (IRT, \cite{embretson2013item}) as a possible solution to this problem. The IRT was developed as a mathematical framework to reason jointly about a student's ability and the questions which they respond to. Considered to be ubiquitous in the field of standardized testing, it is largely used in the design, analysis, and scoring of tests, questionnaires(\cite{hartshorne2018critical}, \cite{harlen2001assessment}, \cite{luniewska2016ratings}), and instruments that measure ability, attitudes, or other latent variables. The IRT allows educators to quantify the ``difficulty" of a given test item by modelling the relationship between a test taker's response to the test item and the test taker's overall ability. In the context of UED, we then see that the IRT provides a useful framework for us to understand the difficulty of a parameterized environment with regards to the ability of the student, of which we aim to maximize.

Our current work proposes a new algorithm, called \textit{Parameterized Environment Response Model}, or PERM. PERM applies the IRT to the UED context and generates curricula by matching environments to the ability of the student. Since we do not use a RL-based teacher or regret as a feedback mechanism, our method is transferable across students, regardless of artificial or human student.

Our main contributions are as follows:
\begin{enumerate}
  \item We propose PERM, a novel framework to assess student ability and difficulty of parameterized environments.
  \item PERM produces curricula by generating environments that matches the ability of the student.
  \item We investigate PERM's capabilities in modelling the training process with latent representations of difficulty and ability.
  \item We compare agents trained with PERM with other UED methods in parameterized environments.
\end{enumerate}

\begin{figure}[t!]
    \begin{subfigure}[b]{0.5\linewidth}
        {\includegraphics[width = \linewidth]{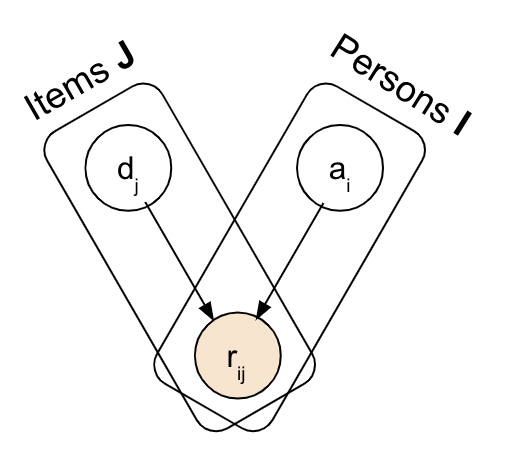}}
        \caption{Item Response Theory}
        \label{fig:IRT_graph}
    \end{subfigure}
    \hfill
    \begin{subfigure}[b]{0.4\linewidth}
        {\includegraphics[width = \linewidth]{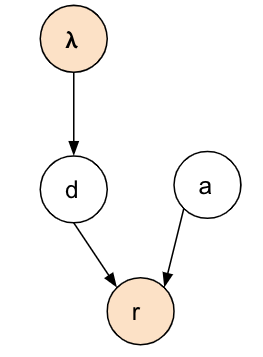}}
        \caption{PERM}
        \label{fig:PERM_graph}
    \end{subfigure}
    \caption{Graphical representation of IRT and PERM. $\lambda$, $a$, $d$, $r$ represents environment parameters, ability, difficulty, and response respectively. White nodes depict latent variables, while tan colored nodes represent observable variables.}
    \label{fig:graph}
\end{figure}

\section{Related Work}

\begin{figure*}
    \includegraphics[width=\textwidth, height = 6cm]{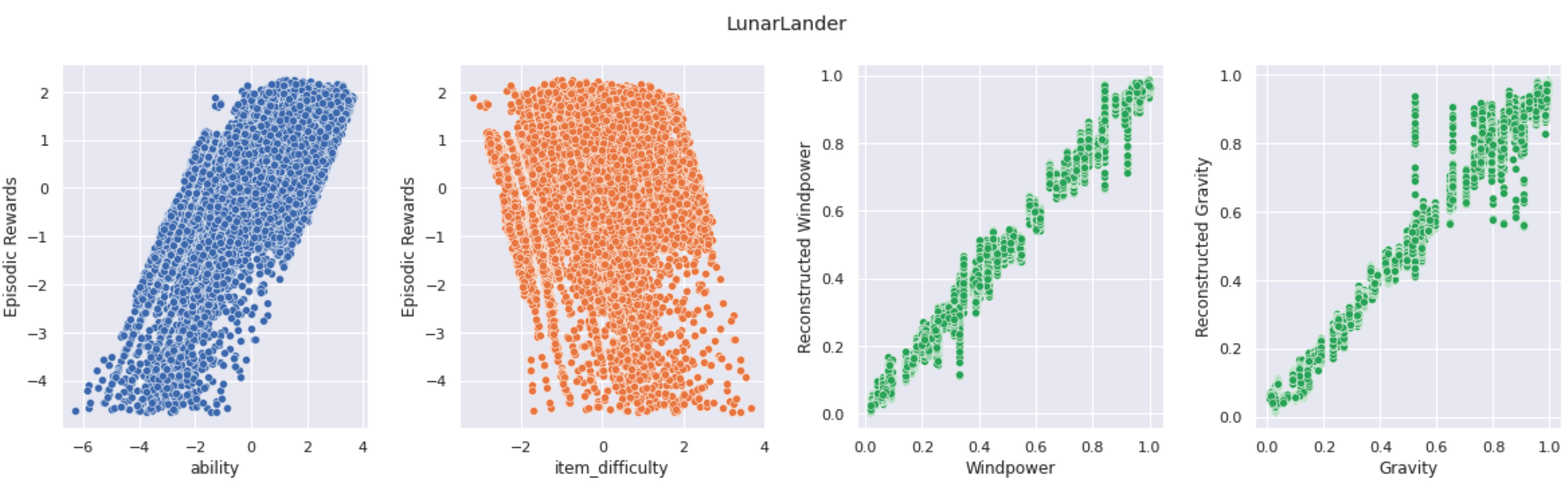}
    \caption{Analysis of PERM's reconstruction capabilities on LunarLander. Blue and orange plots represent ability and difficulty estimates against actual rewards achieved by agent; latent variables learned by PERM correspond to actual reward accordingly. Green plots visualizes the real environment parameters against parameters recovered by PERM. PERM is able to reconstruct the environment parameters from difficulty. Similar results are obtained in BipedalWalker, as seen in Figure \ref{fig:bipedal_analysis}}.
    \label{fig:lunarlander_analysis}
\end{figure*}

\subsection{Item Response Theory}
In Psychology and Education, the IRT is a method used to model interactions between a test taker's ability and a certain characteristic of a question, usually difficulty. The goal here is to gauge a student's ability based on their response to items of varying difficulties. The IRT has many forms, but for the purposes of this paper, we focus on the most standard form: the 1-Parameter Logistic (1PL) IRT, also known as the Rasch model \cite{rasch1993probabilistic}, and a continuous variant. The Rasch model  is given in Eq. \ref{eq:rasch},
\begin{align} \label{eq:rasch}
    p(r_{i,j} = 1 | a_i, d_j) = \frac{1}{1 + \exp^{-(a_i-d_j)}}
\end{align}
where $r_{i,j}$ is the response by the $i$-th person, with an ability measure $a_i$, to the $j$-th item, with a difficulty measure $d_j$. The graph of the 1PL IRT can be seen in Figure \ref{fig:IRT_graph}. We see that the Rasch model is equivalent to a logistic function. Therefore, the probability that a student answers the item correctly is a function of the difference between student ability ${a_i}$ and item difficulty $d_j$.

In RL settings, interactions between agent and environment is summarized by cumulative rewards achieved. For us to adopt the IRT for the environment training scenario, we can then replace the logistic function with a normal ogive model \cite{samejima1974normal}, or the cumulative distribution of a standard normal distribution. Eq \ref{eq:rasch} then becomes:
\begin{align} \label{eq:ogive}
    p(Z \leq r_{i,j} | a_i, d_j) = \frac{1}{\sqrt{2\pi}} \int_{-\infty}^{a_{i} - d_{j}}\exp\{-\frac{u^2}{2}\}du
\end{align}

To our knowledge, the IRT has not used to train RL agents, nor used as a knowledge transfer method from an RL agent to another student.

While earlier works have used different methods to perform inference for IRT, a recent method, VIBO \cite{wu2020variational}, introduces a variational inference approach to estimate IRT. More critically, the formulation of the IRT as a variational inference problem allows us to exploit the learned representation to generate new items. We discuss our modifications to VIBO in Section \ref{methods}.

\subsection{Zone of Proximal Development}
Prior work in UED discusses the zone of proximal development \cite{vygotsky1978mind}, loosely defined as problems faced by the student that are not too easy (such that there is no learning value for the student) and not too difficult (such that it is impossible for the student).

PAIRED \cite{dennis2020emergent} features an adversarial teacher whose task is to generate environments that maximizes the regret between the protagonist student and another antagonist agent. To apply our method to human training, a human-RL pairing would be necessary, but the difference in learning rates and required experiences could create bottlenecks for the human student (i.e. Moravec's Paradox \cite{moravec1988mind}).

PLR \cite{jiang2021prioritized}, and its newer variants (\cite{parker2022evolving}, \cite{jiang2021replay}), maintains a store of previously seen levels and prioritizes the replay levels where the average Generalized Advantage Estimate (GAE, \cite{schulman2015high}) is large. The use of GAE requires access to the value function of the student, a feature that is currently not operationalized for human subjects.

In summary, teacher-student curriculum generation approaches have predominantly focused on the zone of proximal development but have relied on surrogate objectives to operationalize it, without directly measuring difficulty or student ability. However, these surrogate objectives are often non-stationary and not easily transferable between students. Moreover, resulting curricula may not adequately accommodate large changes in student ability, which is a critical limitation for human subjects.

\subsection{Task Sequencing}
Related to UED are previous works in Task Generation and Sequencing for Curriculum Learning, which aim to generate and assign tasks to a student agent in a principled order to optimize performance on a final objective task \cite{narvekar2020curriculum}. Most of the literature in Task Generation focuses on modifying the student agent's MDP to generate different tasks (e.g. \cite{foglino2019optimization}, \cite{narvekar2017autonomous}, \cite{racaniere2019automated}. For example, \textit{promising initialization} \cite{narvekar2016source} modifies the set of initial states and generates a curriculum by initializing agents in states close to high rewards. On the other hand, \textit{action simplification} \cite{narvekar2016source} seeks to prune the action set of the student agent to reduce the likelihood of making mistakes.

In contrast to Task Generation, the UED framework investigates domains where there is no explicit representation of tasks. Here, the student agent must learn to maximize rewards across a variety of environment parameters in open-ended domains without a target ``final task" to learn. In the UED framework, the teacher algorithms only influences the environment parameters $\lambda$, while other features of the student's MDP remains relatively consistent across training. While prior works in task sequencing generate different tasks by directly modifying the student agent's MDP, we leave the curricula generation for such domains for future work.

\section{Method} \label{methods}

In this section we introduce a novel algorithm to train students, combining the training process with an IRT-based environment generator that acts as a teacher. The teacher's goal is to train a robust student agent that performs well across different environment parameters, while the student's goal is to maximize its rewards in a given environment. Unlike previous UED methods which relies on proxy estimates of difficulty and environment feasibility (e.g. Regret \cite{dennis2020emergent}), we propose to directly estimate difficulty by formulating the training process as a student-item response process and model the process with IRT.

Given a teacher model that is able to estimate both ability of the student and difficulty of an item, we are able to present a sequence of environments that are within the student's zone of proximal development. By providing environments within the zone, we are unlikely to generate infeasible environments that are impossible or too difficult for the current student, while also avoiding trivial environments that provide little to no learning value. As our method does not rely on non-stationary objectives such as regret, we are able to train PERM offline and transfer knowledge to any student agent, including human students. Lastly, because our method only relies on environment parameters and student response, it works in virtually any environment without requiring any expert knowledge. We show in later sections that PERM serves as a good exploration mechanism to understand the parameter-response relation given any environment.

PERM can be separated into two components: (i) learning latent representations of ability and difficulty; (ii) generating a curriculum during agent training. The algorithm is summarized in Algorithm \ref{alg:PERMRL}.

\begin{algorithm}[tb]
    \caption{Curriculum Generation for RL Agents with PERM}
    \label{alg:PERMRL}
    \textbf{Input}: Environment $E$, Environment parameters $\lambda$,\
    Student Agent $\pi$\\
    \textbf{Parameter}: $k$ episode frequency before update \\
    \textbf{Output}: Trained Student Agent, Trained PERM
    \begin{algorithmic}[1] 
        \STATE Let $t=0$, $\lambda_0 \sim \textit{Uniform}(\lambda)$\\
        \WHILE{not converged}
        \FOR{$k$ episodes}
        \STATE Collect Reward $r_{t}$ from agent $\pi$ playthrough of $E(\lambda_t)$.
        \STATE Estimate current ability $\mu_{a_t}, \sigma_{a_t}$ by computing $q_{\phi}(a|d,r,\lambda)$\\
        \STATE Sample current ability $a_t \sim N(\mu_{a_t}, \sigma_{a_t})$ \\
        \STATE Get next difficulty $d_{t+1} \gets a_t$ \\
        \STATE Generate next parameters $\lambda_{t+1} \gets p_{\theta}(\lambda|d_{t+1})$ \\
        \STATE $t \gets t + 1$\\
        \ENDFOR \\
        \STATE Update PERM with $\mathcal{L}_{PERM}$ \\
        \STATE RL Update on Student Agent \\
        \ENDWHILE
        \STATE \textbf{return} trained student agent $\pi$, trained PERM
    \end{algorithmic}
\end{algorithm}

\subsection{Preliminaries}
We draw parallels from UED to IRT by characterizing each environment parameter $\lambda$ as an item which the student agent with a policy $\pi_t$ 'responds' to by interacting and maximizes its own reward $r$. Specifically, each student interaction with the parameterized environment yields a tuple $(\pi_t,\lambda_t, r_t)$, where $\pi_t$ represents the student policy at $t$-th interaction, and achieves reward $r_t$ during its interaction with the environment parameterized by $\lambda_t$. We then use a history of such interactions to learn latent representations of student ability $a \in \mathbb{R}^n$ and item difficulty $d \in \mathbb{R}^n$, where $a \propto r$ and $d \propto \frac{1}{r}$. In this formulation, $\pi_t$ at different timesteps are seen as students independent of each other.

\subsection{Learning Latent Representations of Ability and Difficulty}
Following from Wu et al's Variational Item Response Theory (VIBO, \shortcite{wu2020variational}), we use a Variational Inference problem ~\cite{kingma2013auto} formulation to learn latent representation of any student interaction with the environment. More critically, VIBO proposes the amortization of the item and student space, which allows it to scale from discrete observations of items, to a continuous parameter space such as the UED. From here, we drop the subscript for the notations $a$, $d$, and $r$ to indicate our move away from discretized items and students.

\subsubsection{Enabling Generation of Environment Parameters}
The objective of VIBO is to learn the latent representation of student ability and difficulty of items. In order for us to generate the next set of environment parameters $\lambda_{t+1}$for the student to train on, we modify VIBO to include an additional decoder to generate $\lambda$ given a desired difficulty estimate $d$. The graphical form of PERM can be seen in Figure \ref{fig:PERM_graph}.\\

We state and prove the revised PERM objective based on Variational Inference in the following theorem. We use notation consistent with the Variational Inference literature, and refer the motivated reader to \cite{kingma2013auto} for further reading.
\begin{theorem}
    Let $a$ be the ability for any student, and $d$ be the difficulty of any environment parameterized by $\lambda$. Let $r$ be the continuous response from the student on the environment. If we define the PERM objective as

\begin{align} \label{eq:PERM}
    \mathcal{L}_{PERM}& \triangleq \mathcal{L}_{recon_r} + \mathcal{L}_{recon_{\lambda}} + \mathcal{L}_{A} + \mathcal{L}_{D} \\
\intertext{where }
    \mathcal{L}_{recon_r} &= \mathbb{E}_{q_{\phi}(a, d | r, \lambda)} [\log {p_{\theta}(r}| a, d)] \nonumber\\
    \mathcal{L}_{recon_{\lambda}} &= \mathbb{E}_{q_{\phi}(a, d | r, \lambda)} [\log {p_{\theta}(\lambda| d)}] \nonumber\\
    \mathcal{L}_{A} &= \mathbb{E}_{q_{\phi}(a, d | r,\lambda)} [\log \frac{p_(a)}{q_{\phi}(a| d, r, \lambda)}]  \nonumber \\
    &= \mathbb{E}_{q_{\phi}(d | r, \lambda)}[D_{KL}((q_{\phi}(a | d, r)\|p(a))] \nonumber\\
    \mathcal{L}_{D} &= \mathbb{E}_{q_{\phi}(a, d | r, \lambda)} [\log \frac{p(d)}{q_{\phi}(d | r, \lambda)}] \nonumber \\
    &= \mathbb{E}_{q_{\phi}(d | r,\lambda)} [\log \frac{p(d)}{q_{\phi}(d | r, \lambda)}] \nonumber \\
    &= D_{KL}((q_{\phi}(d | r,\lambda)\|p(d))
\intertext{and assume the joint posterior factorizes as follows:}
q_{\phi}(a, d | r, \lambda) &= q_{\phi}(a|d, r, \lambda) q_{\phi}(d| r, \lambda)
\end{align}

then $\log p(r) + \log p(\lambda) \geq \mathcal{L_{PERM}}$ ; $\mathcal{L_{PERM}}$ is a lower bound of the log marginal probability of a response $r$.
\end{theorem}
\begin{proof}
    Expand the marginal and apply Jensen's inequality:
    \begin{align*}
            \log p_{\theta}(r) + &\log p_{\theta}(\lambda) \ge \nonumber
             \mathbb{E}_{q_{\phi}(a, d | r)} [\log \frac{p_{\theta}(r,a,d, \lambda)}{q_{\phi}(a, d | r, \lambda)}] \\
            &= \mathbb{E}_{q_{\phi}(a, d | r}) [\log {p_{\theta}(r| a,d)}] \\
            &+ \mathbb{E}_{q_{\phi}(a, d | r)} [\log {p_{\theta}(\lambda| d)}] \\
            &+ \mathbb{E}_{q_{\phi}(a, d | r)} [\log \frac{p_(a)}{q_{\phi}(a| d , r, \lambda)}] \\
            &+ \mathbb{E}_{q_{\phi}(a, d | r)} [\log \frac{p_(d)}{q_{\phi}(d | r, \lambda)}] \\
            &= \mathcal{L}_{recon_r} + \mathcal{L}_{recon_{\lambda}} + \mathcal{L}_{A} + \mathcal{L}_{D}
    \end{align*}
Since $\mathcal{L}_{PERM}= \mathcal{L}_{recon_r} + \mathcal{L}_{recon_{\lambda}} + \mathcal{L}_{A} + \mathcal{L}_{D}$ and KL divergences are non-negative, we have shown that $\mathcal{L}_{PERM}$ is a lower bound on $\log p_{\theta}(r) + \log p_{\theta}(\lambda)$.

For easy reparameterization, all distributions $q_{\phi}(.|.)$ are defined as Normal distributions with diagonal covariance.
\end{proof}
\begin{figure*}[!htb]
    \includegraphics[width=\textwidth, height = 7cm]{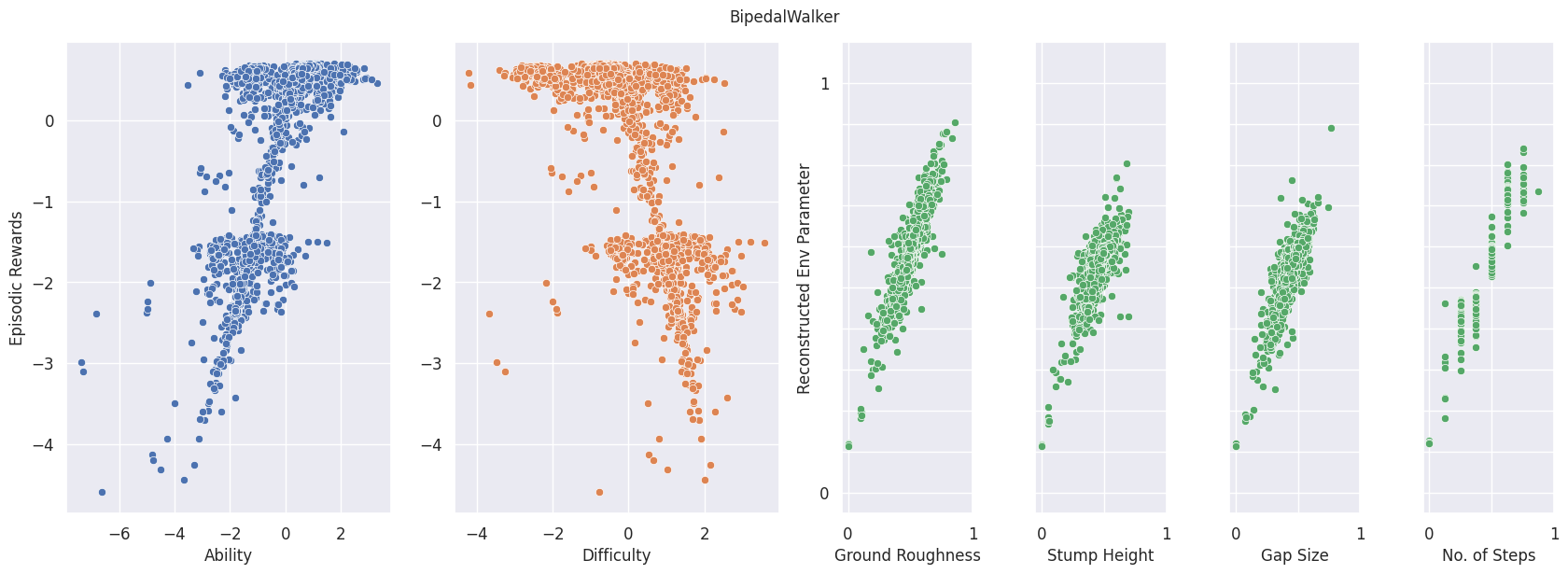}
    \caption{Analysis of PERM's reconstruction capabilities on BipedalWalker. Blue and orange plots represent ability and difficulty estimates against actual rewards achieved by agent; latent variables learned by PERM correspond to actual reward accordingly. Green plots visualizes the real environment parameters against parameters recovered by PERM. PERM is able to reconstruct the environment parameters from difficulty. Values presented are normalized.}
    \label{fig:bipedal_analysis}
\end{figure*}

\subsection{Generating Environments for Curricula}
Our method makes a core assumption that optimal learning takes place when the difficulty of the environment matches the ability of the student. In the continuous response model given in Eq. \ref{eq:ogive}, we see that when ability and difficulty is matched (i.e. $a_i = d_j$), the probability which the student achieves a normalized average score $r_{i,j} = 0$ is 0.5. This is a useful property to operationalize the zone of proximal development, as we can see that the model estimates an equal probability of the student overperforming or underperforming.

Training is initialized by uniformly sampling across the range of environment parameters. After each interaction between the student and the environment, PERM estimates the ability $a_t$ of the student given the episodic rewards and parameters of the environment. PERM then generates the parameters of the next environment $\lambda_{t+1} \sim p_\theta(\lambda|d_{t+1})$ where $d_{t+1} = a_t$.

\section{Experiments}
In our experiments, we seek to answer the following research questions (RQ): \textbf{RQ1}: How well does PERM represent the environment parameter space with ability and difficulty measures? \textbf{RQ2}: How do RL Agents trained by PERM compare to other UED baselines?

We compare two variants of PERM, PERM-Online and PERM-Offline, with the following baselines: {PLR$^\bot$} (Robust Prioritized Replay, \cite{jiang2021replay}), {PAIRED} \cite{dennis2020emergent}, {Domain Randomization}(DR, \cite{tobin2017domain}). PERM-Online is our method that is randomly initialized and trained concurrently with the student agent, as described in Algorithm \ref{alg:PERMRL}; PERM-Offline is trained separately from the student agent, and remains fixed throughout the student training. PERM-Offline is used to investigate its performance when used in an offline manner, similar to how we propose to use it for human training.

For all experiments, we train a student PPO agent \cite{schulman2017proximal} in OpenAI Gym's \textit{LunarLander} and \textit{BipedalWalker} \cite{brockman2016openai}.

We first evaluate PERM's effectiveness in representing the parameter space on both OpenAI environments. Specifically, we evaluate how the latent variables ability $a$ and difficulty $d$ correlate to the rewards obtained in each interaction, as well as its capability in generating environment parameters. We then provide a proof-of-concept of PERM's curriculum generation on the \textit{LunarLander} environment, which has only two environment parameters to tune. Lastly, we scale to the more complex \textit{BipedalWalker} environment that has eight environment parameters, and compare the performance of the trained agent against other methods using the same evaluation environment parameters as in Parker-Holder et al \shortcite{parker2022evolving}.

\subsection{Analyzing PERM’s Representation of Environment Parameters}

\begin{table}[t]
    \begin{tabular}{lrrr}
    \hline
    Env & Response  MSE & $\lambda$ MSE & R-Squared \\
    \hline
    LunarLander   & $7.8 \times10^{-5}$  & 0.001                        & 1.00 \\
    BipedalWalker & $2.5 \times10^{-4}$ & 0.001                        & 0.986 \\
    \hline
    \end{tabular}
    \caption{Analysis of PERM's recovery capabilities. PERM is able to reconstruct the response and environment parameters with great accuracy. R-squared is obtained by regressing ability and difficulty on response.}
    \label{table:perm}
\end{table}
\begin{figure*}[!t]
    \includegraphics[width = \linewidth, height = 0.3\textheight]{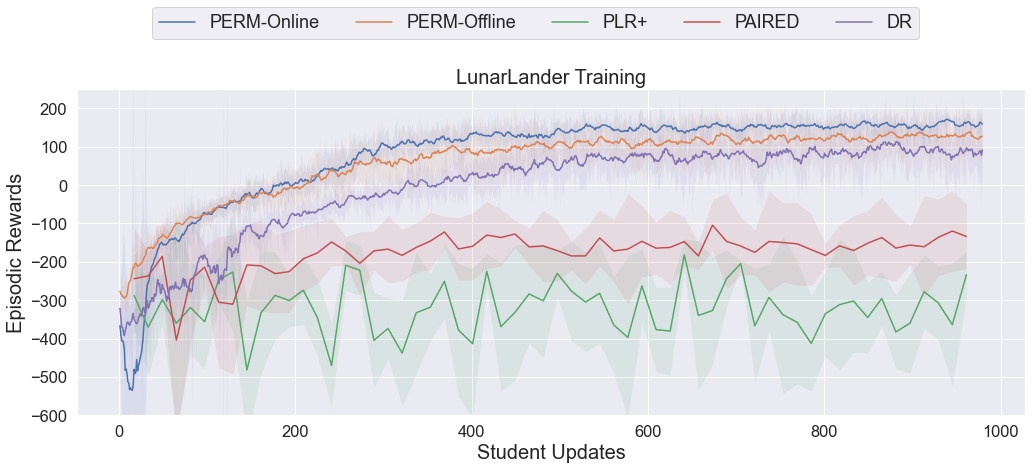}
    \\ \hspace{5mm} \\
    \includegraphics[width = \linewidth, height = 0.3\textheight]{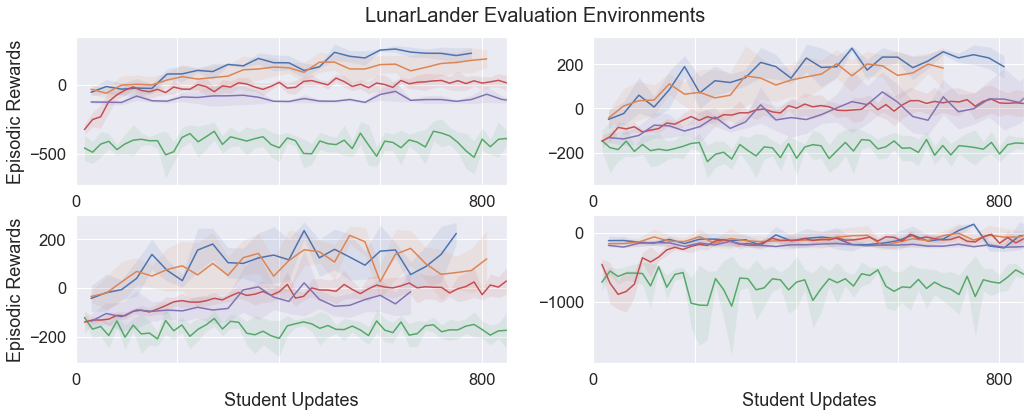}
    \\ \hspace{5mm} \\
    \includegraphics[width = \linewidth, height = 0.3\textheight]{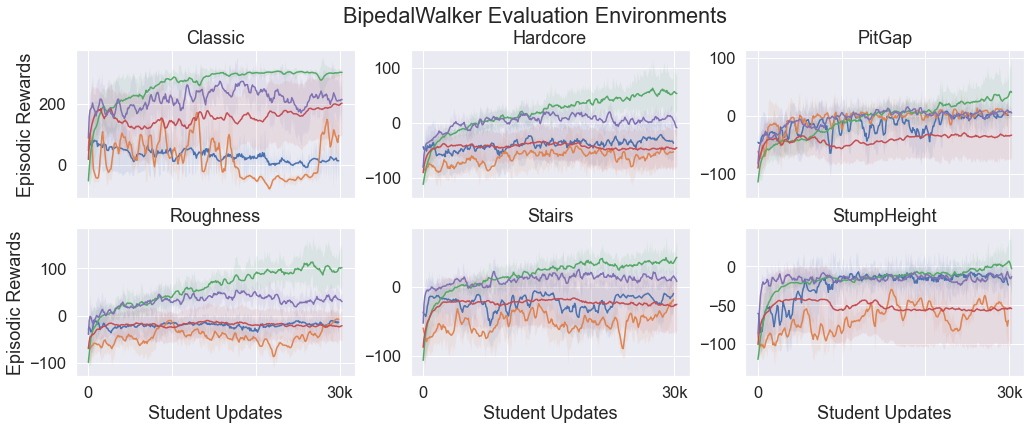}
    \caption{Agents trained by PERM-Online and Offline outperform other methods on LunarLander in both training and evaluation environments. Top: Performance on LunarLander during training; Middle: Performance on selected LunarLander evaluation environments; Bottom: Performance on BipedalWalker evaluation environments.}
    \label{fig:curricula}
\end{figure*}
We begin by investigating PERM's capabilities in representing and generating the environment parameters. In order to establish PERM's capabilities for curricula generation purposes, PERM needs to demonstrate the following: i) the latent representations ability $a$ and difficulty $d$ needs to conform to proposed relationship with response $r$ (i.e. $a \propto r$ and $d \propto \frac{1}{r}$); ii) given input environment parameters $\lambda$ and response $r$, the reconstructed environment parameters $\lambda'$ and response $r'$ need to match its inputs. For both analyses, we rely on correlation metrics and mean-squared error (MSE) to establish PERM's capabilities.

We first train PERM by collecting agent-environment interactions from training a PPO agent under a DR framework until convergence. We then train an offline version of PERM using a subset of data collected and Equation \ref{eq:PERM}. We use the remaining data collected as a holdout set to evaluate PERM's performance. The results are visualized in Figure \ref{fig:lunarlander_analysis} and Figure \ref{fig:bipedal_analysis}, and summary statistics are provided in Table \ref{table:perm}.
As we see in both plots, the latent representations $a$ (blue) and $d$ (orange) largely correlates with our expectations of its respective relationships with the response variable $r$. When both ability and difficulty are regressed against the response variable, we achieve a R-squared of 1.00 and 0.986 for LunarLander and BipedalWalker respectively, indicating that both latent representations are perfect predictors of reward achieved by an agent in a given parameterized environment.
Turning to PERM's capability in generating environment parameters (Figure \ref{fig:lunarlander_analysis} \& \ref{fig:bipedal_analysis}, green), we see that PERM achieves near perfect recovery of all environment parameters on the test set, as indicated by the MSE between input parameters and recovered parameters.
Taking the strong results of PERM in recovering environment parameters from the latent variables, we proceed to generate curricula to train RL Agents.

\subsection{Training RL Agents with PERM}
\subsubsection{LunarLander}

We next apply PERM's environment generation capabilities to train an agent in LunarLander. In this domain, student agents control the engine of a rocket-like structure and is tasked to land the vehicle safely. Before each episode, teacher algorithms determine the gravity and wind power present in a given playthrough, which directly effects the difficulty of landing the vehicle safely. We train student agents for $1e6$ environment timesteps, and periodically evaluate the agent on test environments. The parameters for the test environments are randomly generated and fixed across all evaluations, and are provided in the Appendix. We report the training and evaluation results in Figure \ref{fig:curricula} top and middle plots respectively. As we see, student agents trained with PERM achieves stronger performance over all other methods, both during training and evaluation environments. More importantly, we note that despite training PERM-Offline on a different student, the RL agent under PERM-Offline still maintains its training effectiveness over other methods.

We note that despite a reasonably strong performance of an agent trained under DR, DR has a greater possibility of generating environments that are out of the student's range of ability. We observe that episode lengths for students trained under DR are shorter (mean of 244 timesteps vs 311 timesteps for PERM), indicating a larger proportion of levels where the student agent immediately fails. PERM, by providing environments that are constantly within the student's capabilities, is more sample efficient than DR.

\subsubsection{BipedalWalker}

Finally, we evaluate PERM in the modified BipedalWalker from Parker-Holder et al. \shortcite{parker2022evolving}. In this domain, student agents are required to control a bipedal vehicle and navigate across a terrain. The teacher agent is tasked to select the range of level properties in the terrain, such as the minimum and maximum size of a pit. The environment is then generated by uniformly sampling from the parameters. We train agents for about 3 billion environment steps, and periodically evaluate the agents for about 30 episodes per evaluation environment. The evaluation results are provided in Figure \ref{fig:curricula}, bottom. In the BipedalWalker's evaluation enviroments, student agent trained by PERM produced mixed results, notably achieving comparable performance to PLR$^\bot$ in the \texttt{StumpHeight} and \texttt{PitGap} environment, and comparable performance to PAIRED in others. As BipedalWalker environment properties are sampled from the environment parameters generated by the teacher, it is likely that the buffer-based PLR$^\bot$ that tracks seeds of environments had a superior effect in training our student agents. PERM, on the other hand, is trained to only generate the ranges of the environment properties, which results in non-deterministic environment generation despite the same set of parameters.

\section{Conclusion and Future Work}
We have introduced PERM, a new method that characterizes the agent-environment interaction as a student-item response paradigm. Inspired by Item Response Theory, we provide a method to directly assess the ability of a student agent, and the difficulty associated with parameters of a simulated environment. We proposed to generate curricula by evaluating the ability of the student agent, then generating environments that match the ability of the student. Since PERM does not rely on non-stationary measures of ability such as Regret, our method allows us to predict ability and difficulty directly across different students. Hence, our approach is transferable and is able to adapt to learning trajectories of different students. Theoretically, we could use PERM to train humans in similarly parameterized environments.

We have demonstrated that PERM produces strong representation of both parameterized environments, and is a suitable approach in generating environment parameters with desired difficulties. Finally, we trained RL agents with PERM in our selected environments, and found that our method outperformed the other methods in the deterministic environment, LunarLander.

Most recently, Zhuang et al. \shortcite{zhuang2022fully}  proposed to use a IRT-based model in Computerized Adaptive Testing (CAT) on humans, to some success. The objective of CAT is to accurate predict the student's response to a set of future questions, based on her response to prior questions. We look forward to deploying PERM or IRT-based models in real world settings for training purposes. We hope that our results inspire research into methods that are able to train both humans and RL Agents effectively.

\section*{Acknowledgements}
This research/project is supported by the National Research Foundation, Singapore under its AI Singapore Programme (AISG Award No: AISG-PhD/2022-01-025).

\bibliographystyle{named}
\bibliography{ijcai23}

\end{document}